\documentclass[11pt]{article}

\usepackage{titling}
\usepackage{booktabs} 
\usepackage{amsthm,amsmath,amssymb,tikz}
\usepackage{subcaption}
\usepackage{enumitem}
\usepackage{nicefrac}       
\usepackage{bbm}
\usepackage[margin=1in]{geometry}

\DeclareMathOperator*{\argmax}{arg\,max}

\newcommand{\p}[1]{\left(#1\right)}
\newcommand\numberthis{\addtocounter{equation}{1}\tag{\theequation}}
\newcommand{\row}[2]{\tilde #1_{#2}}
\newcommand{\pH}[2]{\left.\frac{\partial H}{\partial #1}\right|_{#2}}

\def\R{\mathbb{R}}
\def\figscale{1}
\def\zeros{\mathbf{0}}
\def\ones{\mathbf{1}}
\def\tran{^\top}
\def\given{\;|\;}
\def\mc{\mathcal}
\def\sproof{\mc L}
\def\xz{x^*}
\def\vl{\mathbf{\lambda}}
\def\S{\mc S}

\newtheorem{theorem}{Theorem}
\newtheorem*{theorem*}{Theorem} 
\newtheorem{lemma}[theorem]{Lemma}
\newtheorem{corollary}[theorem]{Corollary}
\theoremstyle{definition}
\newtheorem{definition}[theorem]{Definition}

\begin{document}
\title{How Do Classifiers Induce Agents to Invest Effort Strategically?}

\author{
  Jon Kleinberg \\
  Cornell University\\
  \texttt{kleinber@cs.cornell.edu} \\
  \and
  Manish Raghavan \\
  Cornell University\\
  \texttt{manish@cs.cornell.edu}
}
\begin{titlingpage}
\maketitle

\begin{abstract}
  Algorithms are often used to produce decision-making rules that classify
  or evaluate individuals. 
  When these individuals have incentives to be classified a
  certain way, they may behave strategically to influence their outcomes. We
  develop a model for how strategic agents can invest effort in order to change
  the outcomes they receive, and we give a tight characterization of when such
  agents can be incentivized to invest specified forms of effort into improving
  their outcomes as opposed to ``gaming'' the classifier. We show that whenever
  any ``reasonable'' mechanism can do so, a simple linear mechanism suffices.
\end{abstract}
\end{titlingpage}

\section{Introduction}

One of the fundamental insights in the economics of information is the
way in which assessing people (students, job applicants, employees)
can serve two purposes simultaneously: it can identify the strongest
performers, and it can also motivate people to invest effort in
improving their performance \cite{spence1973signaling}.
This principle has only grown in importance with the rise in
algorithmic methods for predicting individual performance
across a wide range of domains, including education, 
employment, and finance.

A key challenge is that we do not generally have access to the 
true underlying properties that we need for an assessment; rather,
they are encoded by an intermediate layer of {\em features}, so that
the true properties determine the features, and the features then
determine our assessment.
Standardized testing in education is a canonical example, in which a test
score serves as a proxy feature for a student's level of learning,
mastery of material, and perhaps other properties 
we are seeking to evaluate as well.
In this case, as in many others, the quantity we wish to measure is
unobservable, or at the very least, difficult to accurately measure; the
observed feature is a construct interposed between the decision rule and the
intended quantity.

This role that features play, as a kind of necessary interface between
the underlying attributes and the decisions that depend on them,
leads to a number of challenges.
In particular, when an individual invests effort to perform better 
on a measure designed by an evaluator,
there is a basic tension between effort invested to
raise the true underlying attributes that the evaluator cares about, and effort
that may serve to improve the proxy features without actually 
improving the underlying attributes.
This tension appears in many contexts --- it is the problem of {\em gaming}
the evaluation rule, and 
it underlies the formulation of {\em Goodhart's Law},
widely known in the economics literature, which
states that once a proxy measure becomes a goal in itself,
it is no longer a useful measure \cite{hardt2016strategic}.
This principle also underpins concerns about strategic gaming 
of evaluations in search engine rankings~\cite{davis2006search}, credit
scoring~\cite{bambauer2018algorithm,foust2008credit}, academic paper
visibility~\cite{beel2009academic}, reputation management~\cite{zarsky2007law},
and many other domains.

\paragraph*{Incentivizing a designated effort investment.} 
These considerations are at the heart of the following class of
design problems, illustrated schematically in 
Figure \ref{fig:graphic}.
An {\em evaluator} creates a decision rule for assessing an {\em agent}
in terms of a set of features, and this leads the agent to make
choices about how to invest effort across their actions
to improve these features.
In many settings, the evaluator views some forms of agent effort
as valuable and others as wasteful or undesirable.
For example, if the agent is a student and the evaluator is
constructing a standardized test, then the evaluator would likely
view it as a good outcome if the existence of the test causes the
student to study and learn the material, but a bad outcome if the
existence of the test causes the student to spend a huge amount of
effort learning idiosyncratic test-taking heuristics specific to the format of 
the test, or to spend effort on cheating.
Similarly, a job applicant (the agent) could prepare for a job interview
given by a potential employer (the evaluator)
either by preparing for and learning material that would directly
improve their job performance (a good outcome for both the agent
and the evaluator), or by superficially memorizing answers to questions
that they find on-line (a less desirable outcome).

Thus, to view an agent's effort in improving their features as
necessarily a form of ``gaming'' is to miss an important subtlety:
some forms of effort correspond intuitively to gaming, while others correspond
to self-improvement.  
If we think of the evaluator as having an opinion on which forms
of agent effort they would like to promote, then 
from the evaluator's point of view,
some decision rules work better than others in creating 
appropriate incentives:
they would like to 
create a decision rule whose incentives lead the agent to invest
in forms of effort that the evaluator considers valuable.

These concerns have long been discussed in the education literature surrounding
the issue of high-stakes standardized testing. In his book ``Measuring Up,''
Daniel Koretz writes,
\begin{quote}
  Test preparation has been the focus of intense argument for many years, and
  all sorts of different terms have been used to describe both good and bad
  forms\dots I think it's best to\dots distinguish between seven different types
  of test preparation: Working more effectively; Teaching more; Working harder;
  Reallocation; Alignment; Coaching; Cheating. The first three are what
  proponents of high-stakes testing want to see~\cite{koretz2008measuring}.
\end{quote}
Because teachers are evaluated based on their students' performance on a test,
they change their behavior in order to improve their outcomes. As Koretz notes,
this can incentivize the investment of both productive and unproductive forms of
effort.

\begin{figure}[t]
  \centering
  \begin{tikzpicture}
    \node[draw,text width=2.5cm,align=center,minimum height=2cm] (eff) at (-5.5, -2)
    {Agent's effort investment \includegraphics[width=2cm]{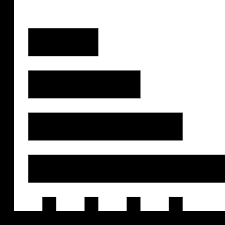}};
    \node[draw,minimum height=2cm,align=left] (feat) at (-2, -2)
    {Features\\\texttt{$F_1$ = ...}\\\texttt{$F_2$ = ...}\\\texttt{$F_3$ = ...}\\\texttt{\dots}};
    \node[draw,text width=2.5cm,align=center,minimum height=2cm] (dec) at (1.5, -2)
    {Evaluator's decision rule \\\includegraphics[width=2.5cm]{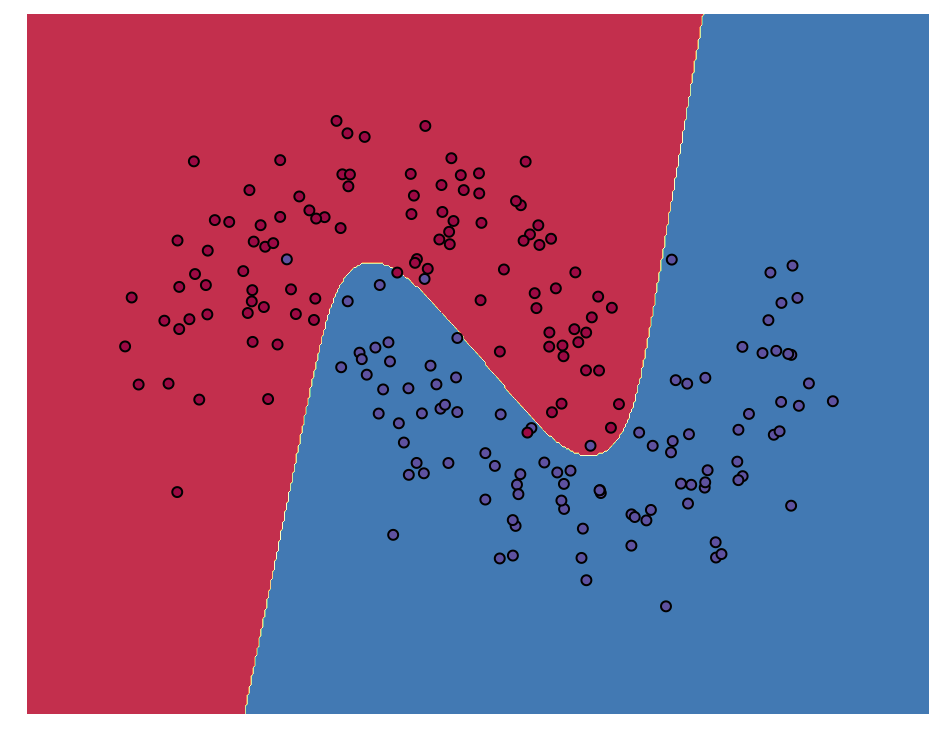}};
    \node (out) at (4.5, -2) {Outcome};
    \draw[->, very thick] (eff) -- (feat);
    \draw[->, very thick] (feat) -- (dec);
    \draw[->, very thick] (dec) -- (out);
  \end{tikzpicture}
  \caption{The basic framework: an agent chooses how
to invest effort to improve the values of certain features, and
an evaluator chooses a decision rule that creates indirect
incentives favoring certain investments of effort over others.
  \label{fig:graphic}
}
\end{figure}
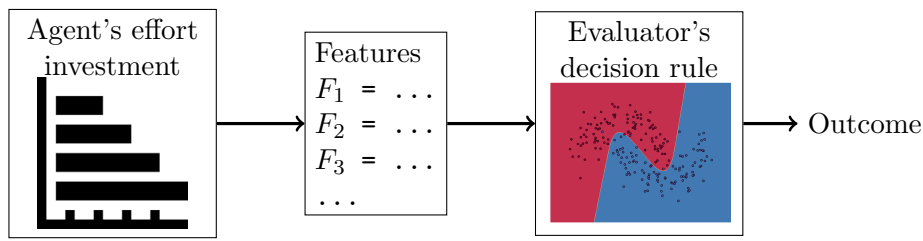

What are the design principles that could help in creating
a decision that incentives the kinds of effort that the
evaluator wants to promote?
Keeping the evaluation rule and the features secret, so as to make
them harder to game, is generally not viewed as a robust solution,
since information about the evaluation process
tends to leak out simply by observing the decisions being made, 
and secrecy can create inequalities between insiders who know
how the system works and outsiders who don't.
Nor should the goal be simply to create a decision rule
that cannot be affected at all by an agent's behavior;
while this eliminates the risk of gaming, it also
eliminates the opportunity for the decision rule to
incentivize behavior that the evaluator views as valuable.

If there were no intermediate features, and the evaluator 
could completely observe an agent's choices about how they
spent their effort across different actions, then the evaluator
could simply reward exactly the actions they want to incentivize.
But when the actions taken by an individual are
hidden, and can be perceived only through an intermediate
layer of proxy features, then the
evaluator cannot necessarily tell whether these features are the result of
effort they intended to promote 
(improving the underlying attribute that the feature
is intended to measure) or effort from other actions that also
affect the feature.
In the presence of these constraints, can one design
evaluation rules that nonetheless incentivize the intended set of behaviors?

To return to our stylized example involving students as agents
and teachers as evaluators, a teacher can choose among many possible
grading schemes to announce to their class; each corresponds to a candidate
decision rule, and each could potentially incentivize different forms
of effort on the part of the students.
For example, the teacher could announce that a certain percentage of 
the total course grade depends on homework scores, and 
the remaining percentage depends on exam scores.
In this context, the homework and the exam scores are
the features that the teacher is able to observe, and the 
students have various actions at their disposal --- studying to
learn material, cheating, or other strategies --- that can improve
these feature values.
How does the way in which the teacher balances the percentage weights
on the different forms of coursework --- producing different
possible decision rules --- affect the decisions students
make about effort?
As we will see in the next section, the model we develop here suggests
some delicate ways in which choices about a decision rule
can in principle have significant effects on agents' decisions about effort.

These effects are not unique to the classroom setting. To take an
example from a very different domain, consider a restaurant trying to improve
its visibility on a review-based platform (e.g. Yelp). 
Here we can think of the platform as the evaluator constructing a decision
rule and the restaurant as the agent: the platform
determines a restaurant's rank based on both the quality of reviews
and the number of users who physically visit it, both of which are
meant to serve as proxies for its overall quality. The restaurant
can individually game either of these metrics by paying people to
write positive reviews or to physically check in to their location, but
improving the quality of the restaurant will ultimately improve both
simultaneously. Thus, the platform may wish to consider both metrics,
rating and popularity, in some balanced way in order to increase the
restaurant's incentive to improve.

\paragraph*{The present work: Designing evaluation rules.}
In this paper, we develop a model for this process of incentivizing effort,
when actions can only be measured through intermediate features.
We cast our model as an interaction between an {\em evaluator}
who is performing
an assessment, and an {\em agent} who wants to score well on this assessment.
An instance of the problem
consists of a set of actions in which the agent can invest
chosen amounts of {\em effort}, and a set of functions determining how
the levels of effort spent on these actions
translate into the values of {\em features} that are observable to
the evaluator.

The evaluator's design task is to create an evaluation rule that takes
the feature values as input, and produces a numerical score as output.
(Crucially, the evaluation rule is not a function of the agent's level of
effort in the actions, only of the feature values.)
The agent's goal is to achieve a high score, and to do this,
they will optimize how they allocate their effort across actions.
The evaluator's goal is to induce a specific {\em effort profile}
from the agent --- specifying a level of effort devoted to each action ---
and the evaluator seeks an evaluation rule that
causes the agent to decide on this effort profile. 
Again, Figure~\ref{fig:graphic} gives a basic view of this pipeline
of activities.

Our main result is a characterization of the instances for which 
the evaluator can create an evaluation rule inducing a specified
effort profile, and a polynomial-time algorithm to construct such a rule
when it is feasible.
As part of our characterization, we find that if there is any 
evaluation rule, monotone in the feature values, that induces the
intended effort profile, then in fact there is one that is linear in
the feature values; and we show how to compute a set of coefficients 
achieving such a rule.
Additionally, we provide a tight characterization of which actions can be
jointly incentivized. 

The crux of our characterization is to consider how an agent is able
to ``convert'' effort from one action to another, or more generally
from one set of actions to another set of actions.
If it is possible to reallocate effort spent on actions the evaluator
is trying to incentivize to actions the evaluator isn't trying to 
incentivize, in a way that improves the agent's feature values, then
it is relatively easy to see that the evaluator won't be able to 
design a decision rule that incentivizes their desired effort profile:
any incentives toward the evaluator's desired effort profile will be undercut
by the fact that this effort can be converted away into other undesired
forms of effort in a way that improves the agent's outcome.
The heart of the result is the converse, providing an if-and-only-if
characterization: when such a conversion by
the agent isn't possible, then we can use the absence of this conversion
to construct an explicit decision rule that incentivizes precisely
the effort profile that the evaluator is seeking.

Building on our main result, we consider a set of further questions as well.
In particular, we discuss characterizations of the set of all
linear evaluation rules that can incentivize a family of allowed
effort profiles,
identifying tractable structure for this set in special cases, but
greater complexity in general.
And we consider the problem of choosing an evaluation rule to 
optimize over a given set of effort profiles,
again identifying tractable special cases and
computational hardness in general.

\paragraph*{Further Related Work.}

Our work is most closely related to the principal-agent literature from
economics: an evaluator (the principal) wants to set a policy (the evaluation
rule) that accounts for the agent's strategic responses. Our main result has
some similarities, as well as some key differences, relative to a classical
economic formulation in principal-agent models
\cite{grossman1983analysis,holmstrom1987aggregation,holmstrom1991multitask,hermalin1991moral}.
We explore this connection in further detail in Section~\ref{sec:pa}.

In the computer science literature, a growing body of work seeks to characterize
the interaction between a decision-making rule and the strategic agents it
governs. This was initially formulated as a zero-sum
game~\cite{dalvi2004adversarial}, e.g. in the case of spam detection, and more
recently in terms of Stackelberg competitions, in which the evaluator publishes
a rule and the agent may respond by manipulating their features
strategically~\cite{hardt2016strategic,bruckner2011stackelberg,dong2018strategic,hu2018disparate,milli2018social}.
This body of work is different from our approach in a crucial respect, in that
it tends to assume that all forms of strategic effort from the agent are
undesirable; in our model, on the other hand, we assume that there are certain
behaviors that the evaluator wants to incentivize.

There is also work on strategyproof linear regression 
\cite{chenstrategyproof,cummings2015truthful,dekel2010incentive}.
The setup of these models is also quite different from ours -- typically, the
strategic agents submit $(x, y)$ pairs where $x$ is fixed and $y$ can be chosen
strategically, and the evaluator's goal is to perform linear regression in a way
that incentivizes truthful reporting of $y$. In our setting, on the other hand,
agents strategically generate their features $x$, and the evaluator rewards them
in some way based on those features.

Work exploring other aspects of how evaluation rules lead to investment of
effort can be found in the economics literature, particularly in the contexts of
hiring~\cite{fryer2013valuing,hu2017short} and affirmative
action~\cite{coate1993will}. While these models tend to focus on decisions
regarding skill acquisition, they broadly consider the investment incentives
created by evaluation. Similar ideas can also be found in the Science and
Technology Studies literature~\cite{ziewitz2018gaming}, considering how
organizations respond to guidelines and regulations.

As noted above, 
principal-agent mechanism design problems in which 
the principal cannot directly
observe the agent's actions have been studied in the economics
literature~\cite{arrow1963uncertainty,pauly1968economics,arrow1968economics},
and include work on the notion of {\em moral hazard.} Insurance markets are
canonical examples in this domain: the agent reduces their liability by
purchasing insurance, and this may lead them to act more recklessly and decrease
welfare. The principal cannot directly observe how carefully the agent is
acting, only whether the agent makes any insurance claims. These models provide
some inspiration for ours; in particular, they are often formalized such that the
agent's actions are ``effort variables'' which, at some cost to the agent,
increase the agent's level of ``production''~\cite{laffont2009theory}. This
could be, for example, acting in more healthy ways or driving more carefully in
the cases of health and car insurance respectively.
Note, however, that in the insurance case, the
agent and the principal have aligned incentives in that both prefer that the
agent doesn't --- e.g., in the case of car insurance --- get
into an accident. In our model,
we make no such assumptions: the agent may have no incentive at all to invest in
the evaluator's intended
forms of effort beyond the utility derived from the mechanism.
The types of scenarios considered in 
insurance markets can be generalized to domains like
share-cropping~\cite{cheung1969theory,stiglitz1974incentives}, corporate
liability~\cite{jensen1976theory}, and theories of
agency~\cite{ross1973economic}. Steven Kerr provides a detailed list of such
instances in his classic paper ``On the folly of rewarding A, while hoping for
B''~\cite{kerr1975folly}.

Concerns over strategic behavior also manifest
in ways that do not necessarily map to
intuitive notions of gaming, but instead where the evaluator does not want to
incentivize the agent to take actions that might be counter to their 
interests.  
For example, Virginia Eubanks considers a case of risk assessment
in the child welfare system; when a risk tool includes features about
a family's history of interaction with public services, 
including aid such as food stamps and public housing,
she argues that it has the potential to incentivize families 
to avoid such services for fear of being labeled high risk 
\cite{Eubanks2018}.
This too would be a case in which
the structure and implementation of an evaluation rule can
incentivize potentially undesirable actions in agents, and would be interesting
to formalize in the language of our model.

\paragraph{Organization of the remainder of the paper.}
Section~\ref{sec:overview} contains all the definitions and technical motivation
leading up to the formulation and statement of our two main results,
Theorems~\ref{thm:main} and~\ref{thm:opt}. Sections~\ref{sec:incent}
and~\ref{sec:opt} contain the proofs of these two results, respectively, and
Section~\ref{sec:lin_space} considers further extensions.

\section{Model and Overview of Results} \label{sec:overview}
\subsection{A Formal Model of Effort Investment}

Here, we develop a formal model of an agent's investment of effort.
There are $m$ actions the agent can take, and they must decide to allocate an
amount of effort $x_j$ to each activity $j$.  We'll assume the agent has some
budget $B$ of effort to invest,\footnote{We might instead model the agent as
  incurring a fixed cost $c$ per unit effort with no budget. In fact, this
  formulation is in a sense equivalent: for every cost $c$, there exists a
  budget $B$ such that an agent with cost $c$ behaves identically to an agent
with fixed budget $B$ (and no cost). For clarity, we will deal only with the
budgeted case, but our results will extend to the case where effort comes at a
cost.} so $\sum_{j=1}^m x_j \leq B$, and we'll call this investment of effort $x
= (x_1, x_2, \ldots, x_m)$ an \textit{effort profile}.

The evaluator cannot directly observe the agent's effort profile, but instead
observes features $F_1, \dots, F_n$ derived from the agent's effort
profile. The value of each $F_i$ grows monotonically in the 
effort the agent invests in certain actions according to an 
\textit{effort conversion function} $f_i(\cdot)$:
\begin{equation}
  F_i = f_i\p{\sum_{j=1}^m \alpha_{ji} x_j},
  \label{eq:Fi_def}
\end{equation}
where each  $f_i(\cdot)$ is nonnegative,
smooth, weakly concave (i.e., actions provide diminishing returns), and strictly
increasing. We assume $\alpha_{ji} \ge 0$.

We represent these parameters of the problem using a bipartite graph with the
actions $x_1, x_2, \ldots, x_m$ on the left, the features $F_1, \ldots, F_n$ on
the right, and an edge of weight $\alpha_{ji}$ whenever $\alpha_{ji} > 0$, so
that effort on action $x_j$ contributes to the value of feature $F_i$. We
call this graph, along with the associated parameters (the matrix $\alpha \in
\R^{m \times n}$ with entries $\alpha_{j i}$; functions $f_i: \R \to \R$ for $i
\in \{1, ..., n\}$; and a budget $B$), the \emph{effort graph} $G$.
Figure~\ref{fig:graph_model} shows some examples of what $G$ might look like.

\def\figwdth{0.4}

\begin{figure}[ht]
  \centering
  \begin{subfigure}{\figwdth\textwidth}
    \centering
    \begin{tikzpicture}[scale=\figscale]
      \node[circle, draw] (x1) at (-1, -.75) {$x_1$};
      \node (d1) at (-1, -1.5) {\vdots};
      \node[circle, draw] (x2) at (-1, -2.5) {$x_j$};
      \node (d2) at (-1, -3.25) {\vdots};
      \node[circle, draw] (x4) at (-1, -4.25) {$x_m$};

      \node[circle, draw] (F1) at (2, -.75) {$F_1$};
      \node[circle, draw] (F2) at (2, -2.5) {$F_2$};
      \node (df) at (2, -3.25) {\vdots};
      \node[circle, draw] (Fn) at (2, -4.25) {$F_n$};

      \node[circle, draw] (H) at (4, -2.5) {$H$};

      \draw[->,very thick] (x1) edge [above] node {$\alpha_{11}$} (F1);
      \draw[->,very thick] (x2) edge [above] node {$\alpha_{j1}$} (F1);
      \draw[->,very thick] (x2) edge [above] node {$\alpha_{j2}$} (F2);
      \draw[->,very thick] (x4) edge [above left] node {$\alpha_{m2}$} (F2);
      \draw[->,very thick] (x4) edge [below] node {$\alpha_{mn}$} (Fn);

      \draw[->,very thick] (F1) edge [above] node {} (H);
      \draw[->,very thick] (F2) edge [above] node {} (H);
      \draw[->,very thick] (Fn) edge [above] node {} (H);
    \end{tikzpicture}
    \caption{General model}
    \label{fig:gen_model}
  \end{subfigure}
  \hspace*{0.1\textwidth}
  \begin{subfigure}{\figwdth\textwidth}
    \centering
    \begin{tikzpicture}[scale=\figscale]
      \node[circle, draw, red] (x1) at (-1, 0.5) {$x_1$};
      \node[circle, draw] (x2) at (-1, -1) {$x_2$};
      \node[circle, draw, red] (x3) at (-1, -2.5) {$x_3$};

      \node[circle, draw] (F1) at (2, 0) {$F_T$};
      \node[circle, draw] (F2) at (2, -2) {$F_W$};

      \node[circle, draw] (H) at (4, -1) {$H$};

      \draw[->,very thick] (x2) edge [below] node {$\alpha_{2T}$} (F1);
      \draw[->,very thick] (x2) edge [above] node {$\alpha_{2W}$} (F2);
      \draw[->,very thick] (x1) edge [above] node {$\alpha_{1T}$} (F1);
      \draw[->,very thick] (x3) edge [below] node {$\alpha_{3W}$} (F2);

      \draw[->,very thick] (F1) edge [above] node {$\beta_T$} (H);
      \draw[->,very thick] (F2) edge [below] node {$\beta_W$} (H);
    \end{tikzpicture}
    \caption{The classroom setting}
    \label{fig:TH}
  \end{subfigure}
  \caption{The conversion of effort to feature values can be represented
using a weighted bipartite graph, where effort $x_j$ spent on action $j$
has an edge of weight $\alpha_{ji}$ to feature $F_i$.
  \label{fig:graph_model}
}
\end{figure}
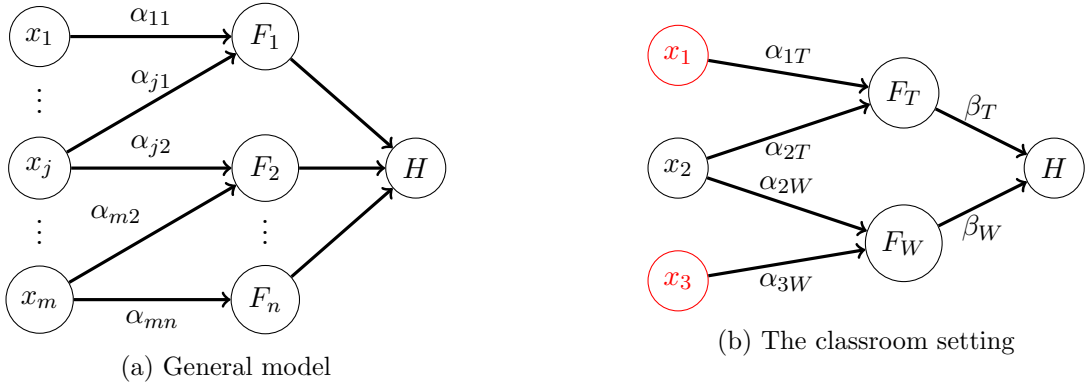

The evaluator combines the features generated by the effort using some mechanism
$M$ to produce an output $H$, which is the agent's utility. 
$M$ is simply a function of the $n$ feature values.
In a classification
setting, for example, $H$ may be binary (whether or not the agent is classified
as positive or negative), or a continuous value (the probability that the agent
receives a positive outcome). Because all features are increasing in the amount
of effort invested by the agent --- in particular, including the
kinds of effort we want to incentivize --- we'll restrict our attention to the
class of monotone mechanisms, meaning that if agent $X$ has larger values in all
features than agent $Y$, then $X$'s outcome should be at least as good as that
of $Y$. Formally, we write this as follows:
\begin{definition}
  A \emph{monotone mechanism} $M$ on features $F_i$ is a mapping $\R^n \to \R$
  such that for $F, F' \in \R^n$ with $F_i' \ge F_i$ for all $i \in \{1, ..., n\}$,
  $M(F') \ge M(F)$. Also, for any $F$, there exists $i \in \{1, ..., n\}$ such that
  strictly increasing $F_i$ strictly increases $M(F)$.
  \label{def:monotone}
\end{definition}
The second of these conditions implies that it is strictly optimal for an agent
to invest all of its budget.
The agent's utility is simply its outcome $H$. Thus, for a mechanism $M$, 
the agent's optimal strategy is to invest effort to maximize
$M(F)$ subject to the constraints that $\sum_{j=1}^m x_j \leq B$ and
$x_j \geq 0$ for all $j$.  (Recall that in this phrasing,
the vector $F$ of feature values is determined from the effort value $x_i$
via the functions $F_i = f_i\p{\sum_{j=1}^m \alpha_{ji} x_j}.$)
We can write the agent's search for an 
optimal strategy succinctly as the following optimization problem:
\begin{align*}
  x^* = \argmax_{x \in \R^m} ~& M(F) \numberthis \label{eq:gen_invest}
  &\text{s.t.} ~ & \sum_{j=1}^m x_j \le B \\
  &&&x \ge \zeros
\end{align*}
where each component $F_i$ of $F$ is defined as in~\eqref{eq:Fi_def}. Throughout
this paper, we'll assume that agents behave rationally and optimally, though it
would be an interesting subject for future work to consider extensions of this
model where agents suffer from behavioral biases. We also note that this is
where we make use of the concavity of the functions $f_i$, since for arbitrary
$f_i$ the agent wouldn't necessarily be able to efficiently solve this
optimization problem.

\subsection{Returning to the classroom example}
To illustrate the use of this
model, consider the effort graph shown in Figure~\ref{fig:TH}, encoding the classroom
example described in the introduction. 
There are two pieces of graded work for the class (a test $F_T$ and
homework $F_W$), and the student can study the material ($x_2$) 
to improve their scores 
on both of these. They can also cheat on the test ($x_1$) and 
look up homework answers on-line ($x_3$).
Their combined effort $\alpha_{1T} x_1 + \alpha_{2T} x_2$ contributes
to their score on the test, and their combined effort
$\alpha_{2W} x_2 + \alpha_{3W} x_3$ contributes to their score on the homework.
To fully specify the effort
graph, we would have to provide a budget $B$ and effort conversion functions
$f_T$ and $f_W$; we leave these uninstantiated, as our main conclusions from
this example will not depend on them.

From these scores, the teacher must decide on a student's final grade $H$. For
simplicity, we'll assume the grading scheme is simply a linear combination,
meaning $H = \beta_T F_T + \beta_W F_W$ for some real numbers $\beta_T, \beta_W
\ge 0$.

The teacher's objective is to incentivize the student to learn the material;
thus, they want to set $\beta_T$ and $\beta_W$ such that the student invests
their entire budget into $x_2$. Of course, this may not be possible. For
example, if $\alpha_{1T}$ and $\alpha_{3W}$ are significantly larger than
$\alpha_{2T}$ and $\alpha_{2W}$ respectively, so that it is much easier to cheat
on the test and copy homework answers than to study, the student would maximize
their utility by investing all of their effort into these undesirable
activities.

In fact, we can make this precise as follows. For any unit of effort invested in
$x_2$, the student could instead invest $\frac{\alpha_{2T}}{\alpha_{1T}}$ and
$\frac{\alpha_{2W}}{\alpha_{3W}}$ units of effort into $x_1$ and $x_3$
respectively without changing the values of $F_T$ and $F_W$. Moreover, if
$\frac{\alpha_{2T}}{\alpha_{1T}} + \frac{\alpha_{2W}}{\alpha_{3W}} < 1$, then
this substitution strictly reduces the sum $x_1 + x_2 + x_3$, leaving
additional effort available (relative to the budget constraint) for raising
the values of $F_T$ and $F_W$.
It follows that in any solution with $x_2 > 0$, there is a way to
strictly improve it through this substitution.
Thus, under this condition, the
teacher cannot incentivize the student to only study.
This is precisely the type of ``conversion'' of effort that we
discussed briefly in the previous section, from 
the evaluator's preferred form of effort ($x_2$) to other forms
($x_1$ and $x_3$)

When $\frac{\alpha_{2T}}{\alpha_{1T}} + \frac{\alpha_{2W}}{\alpha_{3W}} \ge 1$,
on the other hand, a consequence of our results is that that no matter what
$f_T$, $f_W$ and $B$ are, there exist some $\beta_T, \beta_W$ that the teacher
can choose to incentivize the student to invest all their effort into studying.
This may be somewhat surprising -- for instance, consider the case where
$\alpha_{1T} = \alpha_{3W} = 3$ and $\alpha_{2T} = \alpha_{2W} = 2$, meaning
that the best way for the student to maximize their score on each piece of
graded work individually is to invest undesirable effort instead of studying.
Even so, it turns out that the student can still be incentivized to put all of
their effort into studying by appropriately balancing the weight placed on the
two pieces of graded work.

This example illustrates several points that will be useful in what follows.
First, it makes concrete the 
basic obstacle against incentivizing a particular
form of effort: the possibility that it can be
``swapped out'' at a favorable exchange rate for other kinds of effort.
Second, it shows a particular kind of reason why
it might be possible to incentivize a designated form of effort $x_i$:
if investing in $x_i$ improves multiple features simultaneously, the agent might
select it even if it is not the most efficient way to
increase any one feature individually.
This notion of activities that ``transfer'' across different forms of
evaluation, versus activities that fail to transfer, 
arises in the education literature on testing
\cite{koretz1991testing}, and our model shows how such effects
can lead to valuable incentives.

\subsection{Stating the main results} 
In our example, it turned out that a linear grading scheme was sufficient for
the teacher to incentivize the student to study. We formalize such mechanisms as
follows.
\begin{definition}
  A \emph{linear mechanism} $M: \R^n \to \R$ is the mapping $M(F) = \beta\tran
  F = \sum_{i=1}^n \beta_i F_i$ for some $\beta \in \R^n$ such that $\beta_i \ge 0$ for all $i \in \{1, ..., n\}$ and
  $\sum_{i=1}^n \beta_i > 0$.
  \label{def:linear}
\end{definition}
Note that we don't require $\sum_{i=1}^n \beta_i$ to be anything in particular;
the agent's optimal behavior is invariant to scaling $\beta$, so we can
normalize $\beta$ to sum to any intended quantity without affecting the
properties of the mechanism. We rule out the mechanism in which 
all $\beta_i$ are equal to $0$, as it is not a monotone mechanism.

We will say that a mechanism $M$ \textit{incentivizes} effort profile $x$ if $x$
is an optimal response to $M$.
Ultimately, our main result will be to prove the following theorem,
characterizing when a given effort profile can be
incentivized. First, we need to define the support of $x$ as
\begin{equation}
  \S(x) \triangleq \{j \given x_j > 0 \}.
  \label{eq:supp_def}
\end{equation}
With this, we can state the theorem.
\begin{theorem}
  For an effort graph $G$ and an effort profile $x^*$, the following are
  equivalent:
  \begin{enumerate}[itemsep=0ex, topsep=0pt]
    \item There exists a linear mechanism that incentivizes $x^*$.
    \item There exists a monotone mechanism that incentivizes $x^*$.
    \item For all $x$ such that $\S(x) \subseteq \S(x^*)$, there exists a linear
      mechanism that incentivizes $x$. 
  \end{enumerate}
  Furthermore, there is a polynomial time algorithm that decides the
  incentivizability of $x^*$ and provides a linear mechanism $\beta$ to
  incentivize $x^*$ whenever such $\beta$ exists.
  \label{thm:main}
\end{theorem}
When there exists a monotone mechanism incentivizing $x^*$, we'll call both
$x^*$ and $\S(x^*)$ \textit{incentivizable}.\footnote{A closely related notion
in the principal-agent literature is that of an \textit{implementable} action.}
Informally, when $x^*$ is not incentivizable, this algorithm finds a succinct
``obstacle'' to any solution with support $\S(x^*)$, meaning no $x$ such that
$\S(x) = \S(x^*)$ is incentivizable. The following corollary is a direct
consequence of Theorem~\ref{thm:main}. (We use the notation $[m]$ to represent
$\{1, 2, \ldots, m\}$.)
\begin{corollary}
  For a set $S \subseteq [m]$, some $x$ such that $\S(x) = S$ is incentivizable
  if and only if all $x$ with $\S(x) = S$ are incentivizable.
  \label{cor:S_incent}
\end{corollary}

In Section~\ref{sec:incent}, we'll prove Theorem~\ref{thm:main}. The proof we
give is constructive, and it establishes the algorithmic result.

\paragraph*{Optimizing over effort profiles.}
It may be the case that the evaluator doesn't have a single
specific effort profile by the agent that they want to incentivize; 
instead, they may have an objective function defined on effort profiles,
and they would like to maximize this objective function over
effort profiles that are incentivizable.
In other words, the goal is to choose an evaluation rule so that
the resulting effort profile it induces performs as well
as possible according to the objective function.

In Section~\ref{sec:opt}, we consider the following formulation for 
such optimization problems.
We assume that the evaluator wants to maximize a concave function
$g : \R^m \to \R$ over the space of effort profiles, subject to the constraint
that the agent only invests effort in a subset $D \subseteq [m]$ of effort
variables. 
To accomplish this, the evaluator selects an evaluation rule so as to
incentivize an effort profile $x^*$ with $g(x^*)$ as large as possible. This is
what we will mean by optimizing $g$ over the space of effort profiles.
In this setting, we show the following results, which we prove in
Section~\ref{sec:opt}.
\begin{theorem}
  Let $g$ be a concave function over the space of effort profiles,
  and let $D$ be the set of effort variables in which the evaluator
  is willing to allow investment by the agent.
  \begin{enumerate}[itemsep=0ex, topsep=0pt]
    \item If there exists an $x^*$ such that $\S(x^*) = D$ and
      $x^*$ is incentivizable, then any
      concave function $g$ can be maximized over the space of effort profiles
      in polynomial time.
    \item If $|D|$ is constant, then any concave function $g$ 
      can be maximized over the space of effort profiles in polynomial time.
    \item In general, there exist concave functions $g$ that are NP-hard to
      maximize over the space of effort profiles subject to the
      incentivizability condition.
  \end{enumerate}
  \label{thm:opt}
\end{theorem}

In summary, we establish that it is computationally hard to maximize even
concave objectives in general, although as long as the number of distinct
actions the evaluator is willing to incentivize is small, concave objectives can
be efficiently maximized.

The above results characterize optimization over effort profiles; instead, the
evaluator may wish to optimize over the space of mechanisms (e.g., to fit to a
dataset). We consider the feasibility of such optimization in
Section~\ref{sec:lin_space}, showing that the set of linear mechanisms
incentivizing particular actions can be highly nonconvex, making optimization
hard in general.

\subsection{Principal-Agent Models and Linear Contracts} \label{sec:pa}
Now that we have specified the formalism, we are in a position to
compare our model with well-studied 
principal-agent models in economics to see how our results and techniques 
relate to those from prior work. 
In the standard principal-agent setting, the principal's
objective is to incentivize an agent to invest effort in some particular
way~\cite{ross1973economic,grossman1983analysis}. Crucially, the principal
cannot observe the agent's action -- only some outcome that is influenced by the
agent's action. Thus, while the principal cannot directly reward the agent based
on the action it takes, it can instead provide rewards based on outcomes that
are more likely under desirable actions.

To our knowledge, this framework has yet to be applied to settings
based on machine-learning classifiers 
as we do here; and yet, principal-agent models fit quite naturally in
this context. A decision-maker wants to evaluate an agent, which it can only
observe indirectly through features. These features, in turn, reflect the
actions taken by the agent. In this context, the principal offers a ``contract''
by specifying an evaluation rule, to which the agent responds strategically by
investing its effort so as to improve its evaluation. So far, this is in keeping
with the abstract principal-agent
framework~\cite{ross1973economic,grossman1983analysis}.

Moreover, some of the key results we derive echo known results from previous
models, though they also differ in important respects. Linear contracts, in
particular, are often necessary or optimal in principal-agent contexts for a
variety of reasons. In modeling bidding for government contracts, for example,
payment schemes are linear in practice for the sake of simplicity, even though
optimal contracts may be nonlinear~\cite{mcafee1986bidding}. In other models,
contracts are naturally linear because agents maximize reward in expectation
over outcomes generated stochastically from their
actions~\cite{grossman1983analysis}.

Even when they aren't necessitated by practical considerations or modeling
choices, linear contracts have been shown to be optimal in their own right in
some principal-agent models. Holmstr{\"o}m and
Milgrom~\cite{holmstrom1987aggregation,holmstrom1991multitask} consider the
interplay between incentives and risk aversion and characterize optimal
mechanisms in this setting, finding that under a particular form of risk
aversion (exponential utility), linear contracts optimally elicit desired
behavior.  Our models do not incorporate a corresponding
notion of risk aversion, and the role of linear mechanisms 
in our work arises for fundamentally different reasons.

Hermalin and Katz provide a model more similar to ours, in which
observations result stochastically from agents'
actions~\cite{hermalin1991moral}. 
Drawing on basic optimization results that we use here as well
(in particular, duality and Farkas' Lemma), 
they characterize actions as
``implementable'' based on whether they can be in some sense replaced by other
actions at lower cost to the agent. At a high level, we will rely on a similar
strategy to prove Theorem~\ref{thm:main}.  

There are, however, some further fundamental differences between the 
principal-agent models arising from the work of Hermalin and Katz
and the questions and results we pursue here.  
In particular, the canonical models of principal-agent
interaction in economics typically
only have the expressive power to to incentivize a single action, which
stochastically produces a single observed outcome. This basic difference leads
to a set of important distinctions for the modeling goals we have: because our
goal is to incentivize investment over multiple activities given a
multi-dimensional feature vector, with the challenge that different mixtures of
activities can deterministically produce the same feature vector, our model
cannot be captured by these earlier formalisms. 

An important assumption in our model, and in many principal-agent models in
general, is that the principal knows how the agent's effort affects
observations. Recent work has sought to relax this assumption, finding that
linear contracts are optimal even when the principal has incomplete knowledge of
the agent's cost structure~\cite{carroll2015robustness}. It would be an
interesting subject for future work to extend our model so that the principal
does not know or needs to learn the agent's cost structure.

\section{Incentivizing Particular Effort Profiles} \label{sec:incent}
In this section, we develop a tight characterization of which effort profiles
can be incentivized and find linear mechanisms that do so. For simplicity, we'll
begin with the special case where the effort profile to be incentivized is $x^*
= B \cdot e_j$, with $e_j$ representing the unit vector in coordinate $j$ ---
that is, the entire budget is invested in effort on action $j$. Then,
we'll apply the insights from this case to the general case.

\paragraph*{The special case where $|\S(x^*)| = 1$.}
Recall that in the classroom example, the tipping point for when the intended
effort profile could be incentivized hinged on the question of
\textit{substitutability}: the rate at which undesirable effort could be
substituted for the intended effort. We'll characterize this rate as the solution
to a linear program. In an effort graph $G$, recall that $\alpha \in
\R^{m \times n}$ is the matrix with entries $\alpha_{j i}$. Let $\row{\alpha}{j}
\in \R^n$ be the $j$th row of $\alpha$. Then, we'll define the substitutability
of $x_j$ to be
\begin{align*}
  \kappa_j \triangleq \min_{y \in \R^m} ~& y\tran \ones \numberthis
  \label{eq:dual}
  &\text{s.t.} ~& \alpha\tran y \ge \row{\alpha}{j} \\
  &&& y \ge \zeros
\end{align*}
Intuitively, $y$ is a redistribution of effort out of $x_j$ that weakly
increases all feature values.
Note that $\kappa_j \le 1$ because the solution $y = e_j$ (the vector with $1$
in the $j$th position and $0$ elsewhere) is feasible and has value 1.
In Lemma~\ref{lem:subst}, we'll use this notion of substitutability
to show that whenever $\kappa_j < 1$, the agent will at optimality put no effort
into $x_j$. Conversely, in Lemma~\ref{lem:L_j}, we'll show that the when $\kappa_j = 1$,
there exists a linear mechanism incentivizing $\beta$ incentivizing the solution
$x^* = B \cdot e_j$.

It might seem odd that this characterization depends only on $\kappa_j$, which
is independent of both the budget $B$ and effort conversion functions $f_i$;
however, the particular mechanisms that incentivize $x^*$ will depend on these.
This will also be true in the general case: whether or not a particular effort
profile can be incentivized will not depend on $B$ or $f_i$, but the exact
mechanisms that do so will.

\begin{lemma}
  If $\kappa_j < 1$, then in any monotone mechanism $M$,
  $x_j^* = 0$.
  \label{lem:subst}
\end{lemma}

\begin{proof}
  Intuitively, this is an argument formalizing substitution: if $\kappa_j < 1$,
  replacing each unit of effort in $x_j$ with $y_k$ units of effort (where $y$
  comes from the optimal solution to~\eqref{eq:dual}) on each $x_k$ for $k\in
  [m]$ weakly increases all of the feature values $F_i$ while making the budget
  constraint slack. Therefore, any solution with $x_j > 0$ cannot be optimal.

  In more detail,
  consider any solution $x$ with $x_j > 0$. We'll begin by showing that the
  agent's utility is at least as high in the solution $x'$ with $x_k' = x_k +
  y_k x_j$ for all $k \ne j$ and $x_j' = y_j x_j$, where $y$ is an optimal
  solution to the linear program in~\eqref{eq:dual}. Note that $y_j \le \kappa_j < 1$, so
  $x'$ is different from $x$.

  We know from the constraint on~\eqref{eq:dual} that $\alpha\tran y \ge
  \row{\alpha}{j}$, and therefore
  \begin{equation}
    \label{eq:subst}
    \sum_{k=1}^m \alpha_{k i} y_k \ge \alpha_{ji}
  \end{equation}
  for all $i$.
  Then, by~\eqref{eq:subst},
  \begin{align*}
    f_i\p{\sum_{k=1}^m \alpha_{k i} x_k}
    &\le f_i\p{\sum_{k\ne j} \alpha_{k i} x_k  + x_j \sum_{k=1}^m
    \alpha_{k i} y_k}
    = f_i\p{\sum_{k=1}^m \alpha_{k i} x_k'}
  \end{align*}
  Thus, the value of each feature weakly increases from $x$ to $x'$, so in any
  monotone mechanism $M$, the agent's utility for $x'$ is at least as high as it
  is for $x$. Moreover, the budget constraint on $x'$ isn't tight, since
  \[
    \sum_{k=1}^m x_k' = \sum_{k \ne j} (x_k + y_k x_j) + y_j x_j =
    \sum_{k \ne j} x_k + x_j \sum_{k=1}^m y_k < \sum_{k=1}^m
    x_k \le B.
  \]
  By the definition of a monotone mechanism, no solution for which the budget
  constraint isn't tight can be optimal, meaning $x'$ is not optimal. This
  implies that $x$ is not optimal.
\end{proof}

Thus, $\kappa_j < 1$ implies that $x_j = 0$ in any optimal solution. All that
remains to show in this special case is the converse: if $\kappa_j = 1$, there
exists $\beta$ that incentivizes the effort profile $x^* = B \cdot e_j$. To do
so, define $A(x) \in \R^{m \times n}$ to be the matrix with entries $[A(x)]_{j
i} = \alpha_{j i} f_i'([\alpha\tran x]_i)$, and define $a_j(x) \in \R^n$ to be
the $j$th row of $A(x)$. Then, we can define the polytope
\begin{equation}
  \label{eq:ptope}
  \sproof_j \triangleq \{\beta \given A(x^*) \beta \le \beta\tran a_j(x^*) \cdot
  \ones\}.
\end{equation}

By construction, $\sproof_j$ is the set of linear mechanisms that incentivize
$x^*$. This is because for all $k \in [m]$, every
$\beta \in \sproof_j$ satisfies
\begin{align*}
  [A(x^*) \beta]_k \le \beta\tran a_j(x^*)
  &\Longleftrightarrow \sum_{i=1}^n \alpha_{ki} \beta_i f_i'([\alpha\tran x^*]_i) \le
  \sum_{i=1}^n \alpha_{ji} \beta_i f_i'([\alpha\tran x^*]_i)
  \Longleftrightarrow \pH{x_k}{x^*} \le \pH{x_j}{x^*}
\end{align*}
By Lemma~\ref{lem:kkt} in Appendix~\ref{app:agent_response}, this implies that
$x^*$ is an optimal agent response to any $\beta \in \sproof_j$. To complete the
proof of this special case of Theorem~\ref{thm:main}, it suffices to show that
$\sproof_j$ is non-empty, which we do via linear programming duality.
\begin{lemma}
  \label{lem:L_j}
  If $\kappa_j = 1$, then $\sproof_j$ is non-empty.
\end{lemma}
\begin{proof}
  Consider the following linear program.
  \begin{align*}
    \max_{\beta \in \R^n} ~& \beta\tran a_j(x^*) \numberthis \label{eq:primal}
    &\text{s.t.} ~& A(x^*) \beta \le \ones \\
    &&& \beta \ge \zeros
  \end{align*}
  Clearly, if~\eqref{eq:primal} has value at least 1, then $\sproof_j$ is
  non-empty because any $\beta$ achieving the optimum is in $\mc
  L_j$ by~\eqref{eq:ptope}. The dual of~\eqref{eq:primal} is
  \begin{align*}
    \min_{y \in \R^m} ~& y\tran \ones \numberthis \label{eq:dual2}
    &\text{s.t.} ~& A(x^*) \tran y \ge a_j(x^*) \\
    &&& y \ge \zeros
  \end{align*}
  We can simplify the constraints on~\eqref{eq:dual2}: for all $i$,
  \begin{align*}
    [A(x^*)\tran y]_i \ge [a_j(x^*)]_i
    &\Longleftrightarrow \sum_{k=1}^m \alpha_{k i} y_k f_i'([\alpha\tran x^*]_i) \ge
    \alpha_{ji} f_i'([\alpha\tran x^*]_i)
    \Longleftrightarrow \sum_{k=1}^m \alpha_{k i} y_k \ge \alpha_{ji}
  \end{align*}
  Thus,~\eqref{eq:dual2} is equivalent to~\eqref{eq:dual}, which has value
  $\kappa_j = 1$ by assumption. By duality,~\eqref{eq:primal} also has value
  $\kappa_j = 1$, meaning $\sproof_j$ is non-empty.
\end{proof}
We have shown that if $\kappa_j = 1$, then any $\beta \in \sproof_j$
incentivizes $x^*$. Otherwise, by Lemma~\ref{lem:subst}, there are no monotone
mechanisms that incentivize $x^*$. Next, we'll generalize these ideas to prove
Theorem~\ref{thm:main}.

\paragraph*{The general case.}
We'll proceed by defining the analogue of $\kappa_j$ in the case where the
effort profile to be incentivized has support on more than one component.
Drawing upon the reasoning in Lemmas~\ref{lem:subst} and~\ref{lem:L_j}, we'll
prove Theorem~\ref{thm:main}.

Consider an arbitrary effort profile $x^*$ such that $\sum_{i=1}^m x_j^* = B$,
and let $\S(x^*)$ be the support of $x^*$. Let $\alpha_S$ be $\alpha$ with the
rows not indexed by $S$ zeroed out, i.e., $[\alpha_S]_{ji} = \alpha_{ji}$ if $j
\in S$ and 0 otherwise. Let $\ones_S$ be the vector with a 1 for every $j \in S$
and 0 everywhere else, so $\ones_S = \sum_{j \in S} e_j$. Similarly to how we
defined $\kappa_j$, define
\begin{align*}
  \kappa_S \triangleq \min_{y \in \R^m,z \in \R^m} ~& y\tran \ones
  \numberthis \label{eq:dualset} &
  \text{s.t.} ~& \alpha\tran y \ge \alpha_S\tran z \\
  &&& z\tran \ones_S \ge 1 \\
  &&& y,z \ge \zeros
\end{align*}
Intuitively, we can think of the effort given by $z$ as being substituted out
and replaced by $y$.
Note that $\kappa_S \le \min_{j \in S} \kappa_j$, because the special case where
$z_j = 1$ yields~\eqref{eq:dual}. In a generalization of Lemma~\ref{lem:subst},
we'll show that $\kappa_S < 1$ implies that no optimal solution has $x_j > 0$
for all $j \in S$. Lemma~\ref{lem:subst} formalized an argument based on
substitutability, in which the effort invested on a particular node could be
moved to other nodes while only improving the agent's utility. We generalize
this to the case when effort invested on a subset of the nodes can be replaced
by moving that effort elsewhere.
\begin{lemma}
  For any $S \subseteq [m]$, if $\kappa_S < 1$, then any effort profile $x$ such
  that $x_j > 0$ for all $j \in S$ cannot be optimal.
  \label{lem:subst2}
\end{lemma}
\begin{proof}
  The following proof builds on that of Lemma~\ref{lem:subst}. Let $y$ and $z$
  be optimal solutions to~\eqref{eq:dualset}. We know that for all $i$,
  \begin{equation}
    \label{eq:subst2}
    \sum_{j=1}^m \alpha_{j i} y_j \ge \sum_{j \in S} \alpha_{ji} z_j
  \end{equation}
  Let $c \triangleq \min_{j \in S} x_j/z_j$. Note that $c > 0$ because by
  assumption, $x_j > 0$ for all $j \in S$. It is well-defined because $z\tran
  \ones_S \ge 1$ and $z \ge \zeros$, so $z_j$ is strictly positive for some $j \in
  S$. By this definition, $x_j - c z_j \ge 0$ for all $j \in S$.

  We'll again define another solution $x'$ with utility at least as high as $x$,
  but with the budget constraint slack. For all $i$,
  \begin{align*}
    [\alpha\tran x]_i &=
    \sum_{j=1}^m \alpha_{j i} x_j \\
    &= \sum_{j \notin S} \alpha_{j i} x_j  + \sum_{j \in S}
    \alpha_{j i} x_j \\
    &= \sum_{j \notin S} \alpha_{j i} x_j  + \sum_{j \in S}
    \alpha_{j i} (x_j - cz_j) + c \sum_{j \in S} \alpha_{ji} z_j \\
    &\le \sum_{j \notin S} \alpha_{j i} x_j  + \sum_{j \in S}
    \alpha_{j i} (x_j - cz_j) + c \sum_{j=1}^m \alpha_{ji} y_j
    \tag{By~\eqref{eq:subst2}} \\
    &= \sum_{j \notin S} \alpha_{j i} (x_j + cy_j)  + \sum_{j \in S}
    \alpha_{j i} (x_j + c(y_j-z_j)) \\
    &\triangleq [\alpha\tran x']_i,
  \end{align*}
  where we have defined
  \[
    x_j' \triangleq \begin{cases}
      x_j + c y_j & j \notin S \\
      x_j + c (y_j - z_j) & z \in S
    \end{cases}.
  \]
  Because $x_j - c z_j \ge 0$ for all $j \in S$, $x'$ is a valid effort profile.
  Since $f_i$ is increasing, $f_i([\alpha\tran x]_i) \le f_i([\alpha\tran
  x']_i)$. However,
  \[
    \sum_{i=1}^m x_j' = \sum_{j \notin S} x_j + cy_j + \sum_{j \in S} x_j + c
    (y_j - z_j) = x\tran \ones + c (y\tran \ones - z\tran \ones_S) < B.
  \]
  Thus, the budget constraint for $x'$ is not tight, and so for any monotone
  mechanism, there exists a solution $x''$ which is strictly better than $x'$
  and $x$, meaning $x$ is not optimal.
\end{proof}

Lemma~\ref{lem:subst2} tells us which subsets of variables definitely can't be
jointly incentivized. However, given a subset of variables, it doesn't a priori
tell us if these variables \textit{can} be jointly incentivized, and if so,
which particular effort profiles on these variables are incentivizable. In fact,
we'll show that any $x^*$ such that $\kappa_{\S(x^*)} = 1$ is incentivizable.

\begin{lemma}
  Define
  \begin{equation}
    \label{eq:Lx_def}
    \sproof(x) \triangleq \{\beta \given A(x) \beta \le
    \frac{1}{B} {x}\tran A(x) \beta \cdot \ones\}
  \end{equation}
  If $\kappa_{\S(x^*)} = 1$, then $\sproof(x^*)$ is the set of
  linear mechanisms that incentivize $x^*$, and $\sproof(x^*)$ is non-empty.
  \label{lem:L_S}
\end{lemma}
\begin{proof}
  Let $S = \S(x^*)$.
  We know that for any $z$ such that $z \tran \ones_S \ge 1$,
  \begin{align*}
    \kappa_S \le \kappa_S(z) \triangleq \min_{y \in \R^m} ~& y\tran \ones
    \numberthis \label{eq:kSz}
    &\text{s.t.} ~& \alpha\tran y \ge \alpha_S\tran z \\
    &&& y \ge \zeros
  \end{align*}
  because we've just written~\eqref{eq:dualset} without allowing for optimization
  over $z$. Therefore, if $\kappa_S = 1$, then $\kappa_S(z) = 1$ for any $z$.
  We can write each constraint $[\alpha\tran y]_i \ge
  [\alpha_S\tran z]_i$ as
  \begin{align*}
    [\alpha\tran y]_i \ge [\alpha_S\tran z]_i
    &\Longleftrightarrow \sum_{j=1}^m \alpha_{ji} y_j \ge \sum_{j \in S}
    \alpha_{ji} z_j \\
    &\Longleftrightarrow \sum_{j=1}^m \alpha_{ji} f_i'([\alpha\tran x^*]_i) y_j \ge
    \sum_{j \in S} \alpha_{ji} f_i'([\alpha\tran x^*]_i) z_j \\
    &\Longleftrightarrow \sum_{j=1}^m [A(x^*)]_{ji}\tran y_j \ge
    \sum_{j \in S} [A(x^*)]_{ji} z_j
  \end{align*}
  Thus,~\eqref{eq:kSz} is equivalent to the following optimization, where
  similarly to the definition of $\alpha_S$, we define $A_S(x)$ to be $A(x)$
  with all rows $j \notin S$ zeroed out.
  \begin{align*}
    \kappa_S(z) = \min_{y \in \R^m} ~& y\tran \ones
    \numberthis \label{eq:kSz_alt}
    &\text{s.t.} ~& A(x^*) \tran y \ge A_S(x^*) \tran z \\
    &&& y \ge \zeros
  \end{align*}
  The dual of~\eqref{eq:kSz_alt} is
  \begin{align*}
    \eta(z) \triangleq \max_{\beta \in \R^n} ~& \beta\tran (A_S(x^*)\tran z)
    \numberthis \label{eq:primalset}
    &\text{s.t.} ~& A(x^*) \beta \le \ones \\
    &&& \beta \ge \zeros
  \end{align*}
  Thus,~\eqref{eq:primalset} has value $\eta(z) = \kappa_S(z) = 1$. Recall that
  \[
    \sproof(x^*) = \{\beta \given A(x^*) \beta \le
    \frac{1}{B} {x^*}\tran A(x^*) \beta \cdot \ones\}.
  \]
  Clearly, $\sproof(x^*)$ is non-empty because plugging in $z =
  \frac{x^*}{B}$,~\eqref{eq:primalset} has value $\eta(z) = 1$, meaning there
  exists $\beta$ such that for all $j$,
  \begin{equation}
    \eta\p{\frac{x^*}{B}} = \frac{1}{B} \beta\tran (A_S(x^*)\tran x^*) = 1 \ge
    [A(x^*) \beta]_j \label{eq:Lz_cond}
  \end{equation}
  We'll show that $\beta$ incentivizes the agent to
  invest $\xz$ if and only if $\beta \in \sproof(\xz)$.
  Note that~\eqref{eq:Lz_cond} is true if and only if
  \begin{align*}
    \pH{x_j}{\xz} &\le \sum_{k \in S} \frac{x_k^*}{B} \pH{x_k}{\xz} \tag{$\forall
    j \in [m]$}.
  \end{align*}
  The right hand side is the convex combination of the partial
  derivatives of $H$ with respect to each of the $k \in S$. Since this convex
  combination is at least as large as each partial in the combination, it must be
  the case that all of these partials on the right hand side are equal to one
  another. In other words, this is true if and only if $\pH{x_j}{\xz} =
  \pH{x_{j'}}{\xz}$ for all $j, j' \in S$.

  By Lemma~\ref{lem:kkt} in Appendix~\ref{app:agent_response}, this is true if
  and only if $\xz$ is an optimal effort profile, meaning $\sproof(\xz)$ is
  exactly the set of linear mechanisms that incentivize $\xz$.
\end{proof}

Thus, we've shown Theorem~\ref{thm:main}: for any target effort profile $x^*$,
either $\kappa_{\S(x^*)} = 1$, in which case any $\beta \in \mc L(x^*)$
incentivizes $x^*$, or $\kappa_{\S(x^*)} < 1$, in which case no monotone
mechanism incentivizes $x^*$ by Lemma~\ref{lem:subst2}.

\section{Optimizing other Objectives} \label{sec:opt}
So far, we have given a tight characterization of which effort profiles can be
incentivized. Moreover, we have shown that whenever an effort profile can be
incentivized, we can compute a set of linear mechanisms that do so. However,
this still leaves room for the evaluator to optimize over other preferences. For
instance, perhaps profiles that distribute effort among many activities are more
desirable, or perhaps the evaluator has a more complex utility function over the
agent's effort investment.

In this section, we consider the feasibility of such optimization subject to the
constraints imposed by incentivizability. We show that optimization over effort
profiles is possible in particular instances, but in general, it is
computationally hard to optimize even simple objectives over incentivizable
effort profiles.

\paragraph*{Incentivizing a subset of variables.}
For the remainder of this section, we will assume that the evaluator has a set
of designated effort variables $D \subseteq [m]$, and they want to incentivize
the agent to only invest in effort variables in $D$. Recall that a set of
actions $S$ is incentivizable if and only if $\kappa_S = 1$, where $\kappa_S$ is
defined in~\eqref{eq:dualset}. We define the set system
\begin{equation}
  \label{eq:F_D_def}
  \mc F_D = \{S \subseteq D \given \kappa_S = 1\}
\end{equation}
By Theorem~\ref{thm:main}, $\mc F_D$ gives the sets of effort variables that can
be jointly incentivized. As we will show, a consequence of our results from
Section~\ref{sec:incent} is that $\mc F_D$ is downward-closed, meaning that if
$S \in \mc F_D$, then $S' \in \mc F_D$ for any $S' \subseteq S$.

We begin by characterizing when it is feasible to incentivize some $x$ such that
$\S(x) \subseteq D$. As the following lemma shows, this can be done if and only
if some individual $j \in D$ is incentivizable on its own.
\begin{lemma}
  It is possible to incentivize effort in a subset of a designated set of effort
  nodes $D \subseteq [m]$ if and only if $\max_{j \in D} \kappa_j = 1$.
  \label{lem:incent_subset}
\end{lemma}
\begin{proof}
  The set system $\mc F_D$ is downward closed, since $\kappa_{S \cup \{j\}} = 1$
  implies $\kappa_S = 1$ for all $S, j$. This is because any solution
  to~\eqref{eq:dualset} for $S$ is a solution to~\eqref{eq:dualset} for $S \cup
  \{j\}$, so $\kappa_{S} \ge \kappa_{S \cup \{j\}}$. Therefore, if $x$ is such
  that $\S(x) \subseteq D$ is incentivizable, meaning $\kappa_{\S(x)} = 1$, then
  $\kappa_j = 1$ for all $j \in \S(x)$. If $\kappa_j < 1$ for all $j \in D$,
  then no $x$ such that $\S(x) \subseteq D$ is incentivizable.
\end{proof}
Thus, there exists an incentivizable $x$ such that $\S(x) \subseteq D$ if and
only if there is some $j \in D$ such that the agent can be incentivized to
invest all of its budget into $x_j$.

\paragraph*{Objectives over effort profiles.}
In the remainder of this section, we prove Theorem~\ref{thm:opt}.
Lemma~\ref{lem:incent_subset} shows that if the evaluator wants the agent to
only invest effort into a subset $D$ of effort variables, one solution might be
to simply incentivize them to invest all of their effort into a single $j \in
D$. However, this might not be a satisfactory solution --- the evaluator may
want the agent to engage in a diverse set of actions, or to invest at least some
amount in each designated form of effort. Thus, the evaluator may have some
other objective beyond simply incentivizing the designated forms of effort $D$.

We formalize this as follows: suppose the evaluator has some objective $g : \R^m
\to \R$ over the agent's effort profile $x$, and wants to pick the $x$ that
maximizes $g$ subject to the constraint that $x$ is incentivizable and $\S(x)
\subseteq D$. Formally, this is
\begin{align*}
  \argmax_{x \in \R^m} ~ & g(x)  & \text{s.t.} \hspace{2ex} &\kappa_{\S(x)} = 1
  \numberthis \label{eq:eff_opt} \\
  &&& \S(x) \subseteq D
\end{align*}

To make this more tractable, we assume that $g$ is concave, as it will in
general be hard to optimize arbitrary non-concave functions. We will begin by
showing that this optimization problem is feasible when $\kappa_D = 1$, or
equivalently, when $D \in \mc F_D$. We will extend this to show that when $|D|$
is small,~\eqref{eq:eff_opt} can can be solved. In general, however, we will
show that due to the incentivizability constraint, this is computationally hard.

First, we consider the case where $\kappa_D = 1$. Here, it is possible to find a
mechanism to maximize $g(x)$ because any $x$ in the simplex $\{x \given \sum_{j
\in D} x_j = B\}$ is incentivizable by Theorem~\ref{thm:main}. Thus, the
evaluator could simply maximize $g$ over this simplex to get some effort profile
$x^*$ and find a linear mechanism $\beta$ to incentivize $x^*$. Extending this
idea, if $\kappa_D < 1$ but $|D|$ is small, the evaluator can simply enumerate
all subsets $S \subseteq D$ such that $\kappa_S = 1$, optimize $g$ over each one
separately, and pick the optimal $x^*$ out of all these candidates.

However, in general, it is NP-hard to optimize a number of natural objectives
over the set of incentivizable effort profiles if $\kappa_D < 1$. From
Theorem~\ref{thm:main}, we know that incentivizable effort profiles $x$ can be
described by their support $\S(x)$, which must satisfy $\kappa_{\S(x)} = 1$. The
following lemma shows that this constraint on $x$ makes it difficult to optimize
even simple functions because the family of sets $\mc F_D = \{S \subseteq D
\given \kappa_S = 1\}$ can be used to encode the set of independent sets of an
arbitrary graph. Using this fact, we can show that there exist concave
objectives $g$ that are NP-hard to optimize subject to the incentivizability
constraint.

\begin{lemma}
  Given a graph $G = (V, E)$, there exists an effort graph $G'$ and a set of
  designated effort nodes $D$ such that $S \subseteq D$ is an independent set of
  $G$ if and only if $\kappa_S = 1$ in $G'$.
  \label{lem:ind_set}
\end{lemma}

\begin{proof}
  We construct a designated effort node for each $v \in V$, so $D = V$. We also
  construct an undesirable effort node for each $e \in E$, so the total number
  of effort nodes is $m = |V| + |E|$. For ease of indexing, we'll refer to the
  designated effort nodes as $x_v$ for $v \in V$ and the remaining effort nodes
  as $x_e$ for $e \in E$.

  We construct a feature $F_v$ for each vertex $v \in V$.
  Then, let $\alpha_{v, v} = 3$ for all $v \in V$ and
  $\alpha_{e,v} = 2$ for all $v \in V$. For each $e \in E$, this creates the
  gadget shown in Figure~\ref{fig:gadget}.
  \begin{figure}[ht]
    \centering
    \begin{tikzpicture}[scale=\figscale]
      \node[circle, draw] (x2) at (-1, -1) {$x_u$};
      \node[circle, draw, red] (x3) at (-1, -2) {$x_e$};
      \node[circle, draw] (x4) at (-1, -3) {$x_v$};

      \node[circle, draw] (F2) at (2, -1.25) {$F_u$};
      \node[circle, draw] (F3) at (2, -2.75) {$F_v$};

      \draw[->,very thick] (x2) edge [above] node {$3$} (F2);
      \draw[->,very thick] (x3) edge [above] node {$2$} (F2);
      \draw[->,very thick] (x3) edge [above] node {$2$} (F3);
      \draw[->,very thick] (x4) edge [below] node {$3$} (F3);

    \end{tikzpicture}
    \caption{Gadget to encode independent sets}
    \label{fig:gadget}
  \end{figure}
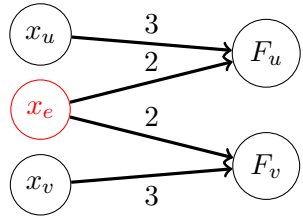

  First, we'll show that if $(u, v) \in E$, then any $S \subseteq D$ containing
  both $u$ and $v$ has $\kappa_S < 1$. Recall the definition of $\kappa_S$
  in~\eqref{eq:dualset}. Consider the solution with $z_u = z_v = \frac{1}{2}$
  and $y_e = \frac{2}{3}$. This is feasible, so $\kappa_S \le \frac{2}{3} < 1$.
  By the contrapositive, if $\kappa_S = 1$ (meaning $S$ is incentivizable), $S$
  cannot contain any $u,v$ such that $(u, v) \in E$, meaning $S$ forms an
  independent set in $G$.

  To show the other direction, consider any independent set $S$ in $G$. By
  construction, $S \subseteq D$ because $D = V$. Then, in the optimal solution
  $(y, z)$ to~\eqref{eq:dualset}, we will show that $y_u = z_u$ for all $u \in
  S$, meaning $\kappa_S = y\tran \ones \ge 1$. To do so, consider the constraint
  $[\alpha\tran y]_u \ge [\alpha_S\tran z]_u$ for any $u \in S$. This is simply
  $3 y_u + 2 \sum_{e = (u, v) \in E} y_e \ge 3z_u$. Because $S$ is an
  independent set, $z_v = 0$ for any $v$ such that $(u, v) \in E$, so this is
  the only constraint in which any such $y_e$ appears. Therefore, it is strictly
  optimal to choose $y_u = z_u$ and $y_e = 0$ for all $e = (u, v) \in E$. As a
  result, $y_u = z_u$ for all $u \in S$, meaning $\kappa_S \ge \sum_{u \in S}
  y_u = \sum_{u \in S} z_u \ge 1$ by the constraint $z\tran \ones_S \ge 1$.
\end{proof}

Thus, if the evaluator wants to find an incentivizable effort profile $x$ such
that $\S(x) \subseteq D$ (the agent only invests in designated forms of effort),
maximizing an objective like $g(x) = \|x\|_0$ (the number of non-zero effort
variables) is NP-hard, due to a reduction from the maximum independent set
problem. Note that $\|x\|_0$ is concave for nonnegative $x$.

Moreover, other simple and natural objectives are hard to optimize as well.
Using a construction similar to the one in Figure~\ref{fig:gadget}, we can
create effort graphs with a set of designated effort nodes $D$ in which $S
\subseteq D$ is incentivizable if and only if $|S| \le k$, meaning $\|x\|_0 \le
k$. This is known to make optimizing even simple quadratic functions (e.g.
$\|{\mc A}x - y\|_2$ for some matrix $\mc A$ and vector $y$)
NP-hard~\cite{natarajan1995sparse}. In general, then, it is difficult to find
the optimal agent effort profile subject to the incentivizability constraint.

\section{The Structure of the Space of Linear Mechanisms} \label{sec:lin_space}
Thus far, we have seen how to construct linear mechanisms that incentivize
particular effort profiles, finding that the mechanisms that do so form a
polytope. Suppose that the evaluator doesn't have a particular effort profile
that they want to incentivize, but instead wants the agent to only invest effort
in a subset of intended effort nodes $D \subseteq [m]$. Generalizing the
definition of $\sproof(x^*)$ as the set of linear mechanisms incentivizing
$x^*$, we define $\sproof(D)$ to be the set of linear mechanisms incentivizing
any $x$ such that $\S(x) \subseteq D$.\footnote{With this notation, we could
write $\sproof_j$ as defined in Section~\ref{sec:incent} as $\sproof(\{j\})$.}
In the remainder of this section, we give
structural results characterizing $\mc L(D)$, showing that in general it can be
highly nonconvex, indicating the richness of the solution space of this problem.

In the simplest case where $|D| = 1$, meaning the evaluator wants to incentivize
a single effort variable, we know by~\eqref{eq:ptope} that $\mc L(D)$ is simply
a polytope. This makes it possible for the evaluator to completely characterize
$\mc L(D)$ and even maximize any concave objective over it.

However, in general, $\mc L(D)$ can display nonconvexities in several ways.
Figure~\ref{fig:gadget} gives an example such that if the evaluator only wants
to incentivize $x_u$ and $x_v$, then $\mc L(D) = \{\beta \given \|\beta\|_0
= 1\}$, meaning $\beta$ has exactly one nonzero entry. This can be generalized
to an example where $\mc L(D) = \{\beta \given \|\beta\|_0 \le k\}$ for any $k$,
which amounts to a nonconvex sparsity constraint.

This form of nonconvexity arises because we're considering mechanisms that
incentivize $x$ such that $\S(x) \subseteq D$. In particular, if $S$ and $S'$
are disjoint subsets of $D$, then we wouldn't necessarily expect the union of
$\mc L(S)$ and $\mc L(S')$ to be convex. However, we might hope that if each
$\mc L(S)$ for $S \subseteq D$ is convex or can be written as the union of
convex sets, then $\mc L(D)$ could also be written as the union of convex sets.

\begin{figure}[ht]
  \centering
  \begin{tikzpicture}[scale=\figscale]
    \node[circle, draw] (x1) at (-1, 0) {$x_1$};
    \node[circle, draw, red] (x2) at (-1, -1) {$x_2$};
    \node[circle, draw] (x3) at (-1, -2) {$x_3$};
    \node[circle, draw, red] (x4) at (-1, -3) {$x_4$};

    \node[circle, draw] (F1) at (2, -.5) {$F_1$};
    \node[circle, draw] (F2) at (2, -1.5) {$F_2$};
    \node[circle, draw] (F3) at (2, -2.5) {$F_3$};

    \node[circle, draw] (H) at (5, -1.5) {$H$};

    \draw[->,very thick] (x1) edge [above] node {$1$} (F1);
    \draw[->,very thick] (x2) edge [above] node {$2$} (F2);
    \draw[->,very thick] (x3) edge [above] node {$1$} (F2);
    \draw[->,very thick] (x3) edge [below] node {$1$} (F3);
    \draw[->,very thick] (x4) edge [below] node {$2$} (F3);

    \draw[->,very thick] (F1) edge [above] node {$\beta_1$} (H);
    \draw[->,very thick] (F2) edge [above left] node {$\beta_2$} (H);
    \draw[->,very thick] (F3) edge [below] node {$\beta_3$} (H);
  \end{tikzpicture}
  \caption{Non-convexity of $\sproof^*(D)$}
  \label{fig:nonconvexS}
\end{figure}

Unfortunately, this isn't the case. Let $\mc L^*(D)$ be the set of mechanisms
incentivizing $x$ such such that $\S(x) = D$ (as opposed to $\S(x) \subseteq
D$). $\mc L^*(D)$ may still be nonconvex, depending on the particular effort
conversion functions $f(\cdot)$.
Consider
the effort graph shown in Figure~\ref{fig:nonconvexS} with $B = 1$, $f_1(y) =
f_2(y) = 1 - e^{-y}$ and $f_3(y) = 1 - e^{-2y}$. Let $D = \{1, 3\}$. To
incentivize $x_1 > 0$ and $x_3 > 0$ simultaneously with $x_2 = x_4 = 0$, it must
be the case that
\[
  \pH{x_1}{x} = \beta_1 f_1'(x_1) = \beta_2 f_2'(x_3) + \beta_3 f_3'(x_3) =
  \pH{x_3}{x}.
\]
To incentivize $x_2 = x_4 = 0$, we must also have
\begin{align*}
  \pH{x_2}{x} = 2\beta_2 f_2'(x_3) &\le \beta_2 f_2'(x_3) + \beta_3 f_3'(x_3) =
  \pH{x_3}{x} \\
  \pH{x_4}{x} =  2\beta_3 f_3'(x_3) &\le \beta_2 f_2'(x_3) + \beta_3 f_3'(x_3) =
  \pH{x_3}{x}
\end{align*}
This is only possible if $\beta_2 f_2'(x_3) = \beta_3 f_3'(x_3)$, meaning
$\beta$ incentivizes $x$ such that $\S(x) = \{1, 3\}$ if and only if
\begin{align} \label{eq:nc1}
  \beta_1 f_1'(x_1) &= \beta_2 f_2'(x_3) + \beta_3 f_3'(x_3) \\
  \label{eq:nc2}
  \beta_2 f_2'(x_3) &= \beta_3 f_3'(x_3)
\end{align}
Combining~\eqref{eq:nc1} and~\eqref{eq:nc2}, we get $\beta_1 f_1'(x_1) =
2\beta_2 f_2'(x_3)$, implying
\begin{align*}
  \beta_1 f_1'(x_1) &= 2\beta_2 f_2'(x_3) \\
  \beta_1' e^{-x_1} &= 2 \beta_2 e^{-x_3} \numberthis \label{eq:b12} \\
  \beta_2 &= \frac{\beta_1 e^{x_3-x_1}}{2} \numberthis \label{eq:b2}
\end{align*}
Similarly, we can derive
\begin{equation}
  \beta_3 = \frac{\beta_1 e^{2x_3-x_1}}{4}
  \label{eq:b3}
\end{equation}

We'll show non-convexity by giving two linear mechanisms $\beta$ and
$\beta'$ that both incentivize an $x$ such that $\S(x) = \{1, 3\}$, but $\beta''
= \frac{1}{2}(\beta + \beta')$ does not incentivize such an $x$.

Let $\beta$ and $\beta'$ incentivize $x = [\nicefrac{1}{3} ~ 0 ~ \nicefrac{2}{3}
~ 0]\tran$ and $x' = [\nicefrac{2}{3} ~ 0 ~ \nicefrac{1}{3} ~ 0]\tran$
respectively. Without loss of generality, we can set $\beta_1 = \beta_1' = 1$.
Using~\eqref{eq:b2} and~\eqref{eq:b3}, we get
\begin{align*}
  \beta &= \begin{bmatrix}
    1 & \frac{e^{1/3}}{2} & \frac{e}{4}
  \end{bmatrix}\tran
  \hspace{1in}
  \beta' = \begin{bmatrix}
    1 & \frac{e^{-1/3}}{2} & \frac{1}{4}
  \end{bmatrix}\tran
\end{align*}
Then, let $\beta'' = \frac{1}{2}(\beta + \beta')$. If $\beta''$ incentivizes
$x^*$ such that $\S(x^*) = \{1, 3\}$, then by~\eqref{eq:nc2}, it must be the case that
\begin{align*}
  \beta_2'' f_2'(x^*_3) &= \beta_3'' f_3'(x^*_3) \\
  \beta_2'' e^{-x^*_3} &= 2\beta_3'' e^{-2x^*_3} \\
  e^{x^*_3} &= \frac{2 \beta_3''}{\beta_2''} \\
  x^*_3 &= \log\p{\frac{e+1}{e^{1/3} + e^{-1/3}}} \approx 0.566
\end{align*}
On the other hand, by~\eqref{eq:b12}, we must also have
\begin{align*}
  \beta_1'' f_1'(x^*_1) &= 2 \beta_2'' e^{-x^*_3} \\
  e^{-x^*_1} &= \frac{e^{1/3} + e^{-1/3}}{2} \exp\p{-\log\p{\frac{e+1}{e^{1/3}
  + e^{-1/3}}}} \\
  e^{-x^*_1} &= \frac{e^{1/3} + e^{-1/3}}{2} \cdot \frac{e^{1/3} +
  e^{-1/3}}{e+1} \\
  x^*_1 &= -\log\p{\frac{(e^{1/3} + e^{-1/3})^2}{2(e+1)}} \approx 0.511
\end{align*}
Such a solution would fail to respect the budget constraint (since $x^*_1
+ x^*_3 > 1 = B$), meaning $\beta''$ cannot incentivize $x^*$ such that $\S(x^*)
= \{1, 3\}$. In fact, the above analysis shows that for any $x^*$ incentivized
by $\beta''$, $\S(x^*)$ must include either $2$ or $4$ because $\beta''$
incentivizes neither $x^*_1 = 1$ nor $x^*_3 = 1$, meaning the only way to use
the entire budget is to set $x^*_2 > 0$ or $x^*_4 > 0$. Thus, despite the fact
that both $\beta$ and $\beta'$ incentivize effort profiles with support $\{1,
3\}$, a convex combination of them does not. As a result, the set of linear
mechanisms incentivizing a subset of effort variables may in general exhibit
complex structures that don't lend themselves to simple characterization.

We visualize this nonconvexity in Figure~\ref{fig:noncon}, where for clarity we
modify the effort graph in Figure~\ref{fig:nonconvexS} by setting $\alpha_{22} =
0$. The yellow region corresponds to $(\beta_2, \beta_3)$ values such that
$\beta = (1 ~ \beta_2 ~ \beta_3)^{\top}$ incentivizes $x$ such that $\S(x) =
\{1, 3\}$. Note that the upper left edge of this region is slightly curved,
producing the non-convexity. As a result, the set of linear mechanisms
incentivizing a subset of effort variables may in general exhibit complex
structures that don't lend themselves to simple characterization.

\begin{figure}[ht]
  \centering
  \includegraphics[width=10cm]{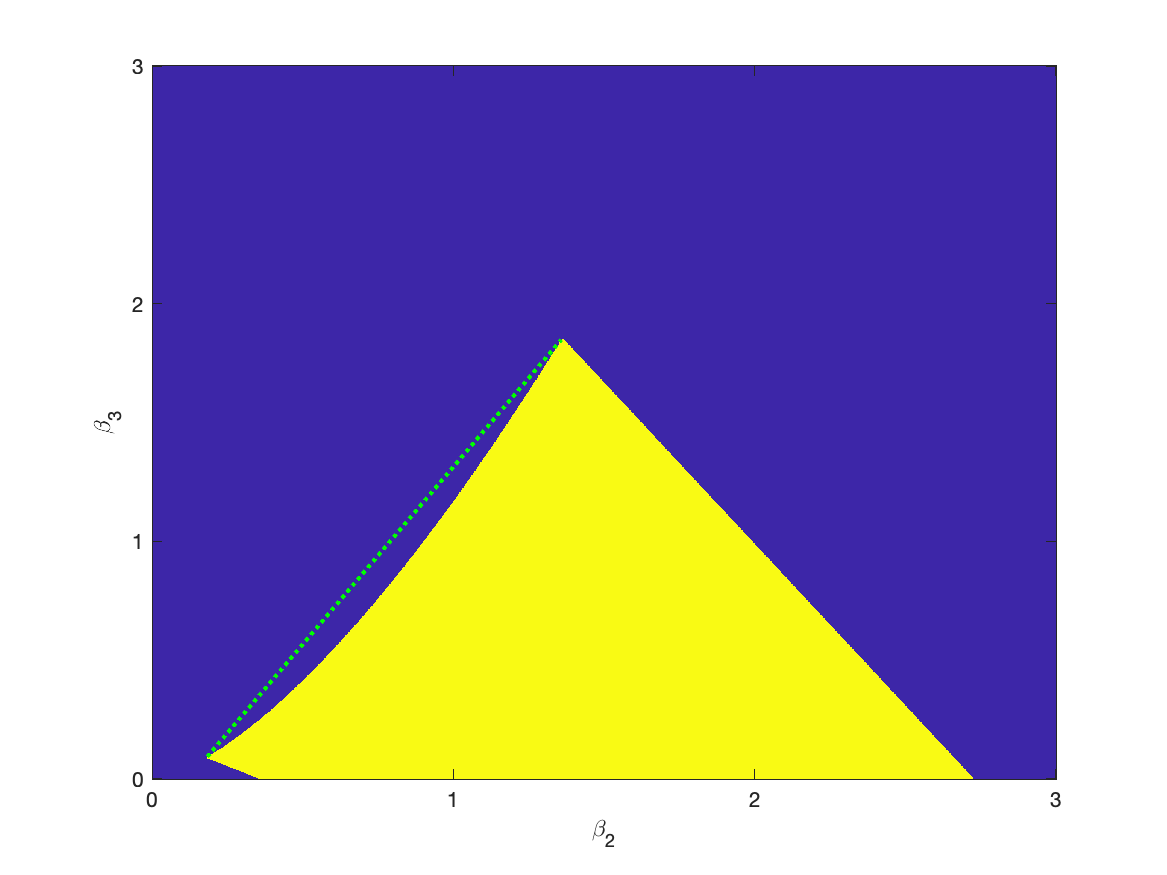}
  \caption{Non-convexity in $(\beta_2, \beta_3)$ pairs}
  \label{fig:noncon}
\end{figure}

\paragraph*{Implications for optimization.}
The complexity of $\mc L(D)$ has immediate hardness implications for
optimizing objectives over the space of linear mechanisms.
For
example, mechanisms that distribute weight on multiple features may be
preferable because in practice, measuring multiple distinct features may lead to
less noisy evaluations.
We might also
consider the case where the evaluator has
historical data $\mc A \in \R^{r \times n}$ and $y \in \R^r$, where each
row of $\mc A$ contains the features $F$ of some individual and each entry of $y$
contains their measured outcome of some sort. Then, in the absence of strategic
considerations, the evaluator could just choose $\beta$ that minimizes squared
error $\|\mc A \beta - y\|_2$ between the scores given by the mechanism and
the outcomes $y$ in the dataset. As noted above, there are examples for which
$\mc L(D) = \{\beta \given \|\beta\|_0 \le k\}$, which is known to make
minimizing squared error NP-hard~\cite{natarajan1995sparse}. However,
in the special case when $D = \{j\}$ (there's only one effort node the evaluator
wants to incentivize), then if $\kappa_j = 1$, the set of linear mechanisms
incentivizing $x^* = B \cdot e_j$ is just the convex polytope $\mc L_j$ defined
in~\eqref{eq:ptope}. Thus, it is possible to maximize any concave objective over
this set.

\section{Conclusion}
Strategic behavior is a major challenge in designing simple and transparent
evaluation mechanisms. In this work, we have developed a model in which
strategic behavior can be directed toward specified forms of effort
through appropriate designs.

Our results leave open a number of interesting questions. All of our analysis
has been for the case in which an evaluator designs a mechanism optimized for
the parameters of a single agent (or for a group of agents who all have
the same parameters). Extending this reasoning to consider the
incentives of a heterogeneous group of agents, where the parameters
differ across members of the group, is a natural further direction. 
In addition, we have assumed throughout that agents behave rationally,
in that they perfectly optimize their allocation of effort.
But it would also be interesting to consider agents with
potential biases that reflect human behavioral principles,
resulting in sub-optimal behavior that follows certain structured properties.
Finally, although we have shown that linear
mechanisms suffice whenever a monotone mechanism can incentivize intended
behavior, if the output of the mechanisms is constrained in some way (e.g.
binary classification), it is an open question to determine
what types of mechanisms are appropriate.

\paragraph*{Acknowledgments.}
Acknowledgments. We thank Rediet Abebe, Solon Barocas, Larry Blume, Fernando
Delgado, Karen Levy, and Helen Nissenbaum for their useful feedback and
suggestions. Thanks to Tal Alon, Magdalen Dobson, and Jamie Tucker-Foltz for
fixing an error in Lemma~\ref{lem:ind_set} that appeared in an earlier version
of this work. This work has been supported in part by a Simons Investigator
Award, a grant from the MacArthur Foundation, graduate fellowships from the
National Science Foundation and Microsoft, and NSF grants CCF1740822 and
SES-1741441.

\bibliographystyle{plain}
\bibliography{refs}

\appendix
\section{Characterizing the Agent's Response to a Linear Mechanism.}
\label{app:agent_response}
In this section, we'll characterize how a rational agent best-responds to a
linear mechanism. Its utility is $H = \beta\tran F$, and therefore we can
rewrite the optimization problem~\eqref{eq:gen_invest} with $M(F) = \beta\tran
F$, which yields
\begin{align*}
  \max_{x \in \R^m} ~& \sum_{i=1}^n \beta_i f_i([\alpha\tran x]_i) \numberthis
  \label{eq:lin_opt} \\
  \text{s.t.} ~& x \ge \zeros \\
  ~& \sum_{j=1}^m x_j \le B
\end{align*}
Note that this is a concave maximization since each $f_i$ is weakly concave and
$[\alpha\tran x]_i$ is linear in $x$. The Lagrangian is
then
\[
  \sproof(x, \vl) = \sum_{i=1}^n \beta_i f_i\p{[\alpha\tran x]_i} +
  \lambda_0\p{B - \sum_{j=1}^m x_j} + \sum_{j=1}^m \lambda_j x_j.
\]
By the Karush-Kuhn-Tucker conditions, since~\eqref{eq:lin_opt} is convex, a
solution $x^*$ is optimal if and only if $\nabla_x \sproof(x^*, \vl^*) =
\mathbf{0}$, so for each $j \in [m]$,
\[
  \sum_{i=1}^n \alpha_{ji} \beta_i f_i'\p{[\alpha\tran x^*]_i} - \lambda_0^* +
  \lambda_j^* = 0.
\]
Note that we can write this as
\[
  \lambda_0^* = \left.\frac{\partial H}{\partial x_j}\right|_{x^*} +
    \lambda_j^*.
\]
By complementary slackness, $\lambda_j^* > 0 \Longrightarrow x_j^* = 0$.
Therefore, it follows that at optimality, the gradients with respect to all
nonzero effort components are $\lambda_0^*$. Furthermore, the gradients with
respect to all effort components are at most $\lambda_0^*$ since $\lambda_j^*
\ge 0$ by definition. This proves the following lemma.
\begin{lemma}
  For any $x \in \R^m$ such that $x \ge \zeros$, $x$
  is an optimal solution to~\eqref{eq:lin_opt} if and only if the following
  conditions hold
  \begin{enumerate}
    \item $\sum_{j=1}^m x_j = B$
    \item For all $j, j'$ such that $x_j > 0$ and $x_{j'} > 0$,
      \[
        \left.\frac{\partial H}{\partial x_j}\right|_x = \left.\frac{\partial
        H}{\partial x_{j'}}\right|_x
      \]
    \item For all $j$ such that $x_j > 0$ and for all $j'$,
      \[
        \left.\frac{\partial H}{\partial x_j}\right|_x \ge \left.\frac{\partial
        H}{\partial x_{j'}}\right|_x
      \]
  \end{enumerate}
  \label{lem:kkt}
\end{lemma}
\begin{proof}
  Choose $\lambda_0^* = \left.\frac{\partial H}{\partial x_j}\right|_x$ for any
  $j$ such that $x_j > 0$. Choose $\lambda_j^* = \lambda_0^* -
  \left.\frac{\partial H}{\partial x_j}\right|_x$ for all $j$. Then, $(x,
  \lambda^*)$ satisfies stationarity (since $\nabla_x \sproof(x, \lambda^*) =
  \zeros)$, primal and dual feasibility by definition, and complementary
  slackness (since $B - \sum_{j=1}^m x_j = 0$). Therefore, $x$ is an optimal
  solution to~\eqref{eq:lin_opt}.

  To show the other direction, note that $\max_j \frac{\partial H}{\partial x_j}
  > 0$ because each $f_i(\cdot)$ is strictly increasing and there is some
  nonzero $\beta_i$. Therefore, $\lambda_0 > 0$, and by complementary slackness,
  every optimal solution must satisfy $\sum_{j=1}^m x_j = B$.
\end{proof}

\end{document}